\documentclass{article}

\usepackage[numbers,compress]{natbib}

\usepackage[margin=1.25in]{geometry}

\usepackage[utf8]{inputenc} 
\usepackage[T1]{fontenc}    
\usepackage{hyperref}       
\usepackage{url}            
\usepackage{booktabs}       
\usepackage{amsfonts}       
\usepackage{nicefrac}       
\usepackage{microtype}      
\usepackage{xcolor}         

\usepackage{wrapfig}  
\usepackage{graphicx}
\usepackage{subfigure}

\usepackage{algorithm}
\usepackage{algpseudocode}
\usepackage{amsmath}
\usepackage{amssymb}
\usepackage{amsthm}
\usepackage{hyperref}
\usepackage{enumitem}

\newtheorem{lemma}{Lemma}[section]
\newtheorem{theorem}[lemma]{Theorem}

\newtheorem{definition}[lemma]{Definition}
\newtheorem{corollary}[lemma]{Corollary}
\newtheorem{proposition}[lemma]{Proposition}
\newtheorem{claim}[lemma]{Claim}

\newcommand{\cD}{\mathcal{D}}

\newcommand{\cF}{\mathcal{F}}

\newcommand{\cH}{\mathcal{H}}
\newcommand{\cI}{\mathcal{I}}

\newcommand{\cW}{\mathcal{W}}
\newcommand{\cX}{\mathcal{X}}

\newcommand{\opt}{\mathrm{OPT}}

\newcommand{\eps}{\epsilon}
\newcommand{\yhat}{\hat{y}}
\newcommand{\what}{\hat{w}}
\newcommand{\spred}{\hat{\sigma}}

\newcommand{\E}{\mathbb{E}}
\newcommand{\R}{\mathbb{R}}
\newcommand{\N}{\mathbb{N}}
\newcommand{\Z}{\mathbb{Z}}

\DeclareMathOperator{\poly}{poly}

\DeclareMathOperator{\sgn}{sgn}

\newcommand{\bone}{\textsf{1}}

\newcommand{\cost}{\textsc{cost}}

\def\doctype{1}

\def\tsubmission{2}
\ifnum\doctype=\tsubmission
	\newcommand{\full}[1]{}
	\newcommand{\submit}[1]{#1}
\else
	\newcommand{\full}[1]{#1}
	\newcommand{\submit}[1]{}
\fi

\title{Algorithms with Prediction Portfolios}

\author{Michael Dinitz\thanks{Department of Computer Science, Johns Hopkins University, Baltimore, MD.  \texttt{mdinitz@cs.jhu.edu}.  Supported in part by NSF grant CCF-1909111.} \and Sungjin Im\thanks{Electrical Engineering and Computer Science, University of California, 5200 N. Lake Road, Merced CA 95344. \texttt{sim3@ucmerced.edu}. Supported in part by NSF grants CCF-1617653, CCF-1844939 and CCF-2121745.} \and Thomas Lavastida\thanks{Jindal School of Management, University of Texas at Dallas, Richardson, TX.  \texttt{thomas.lavastida@utdallas.edu}.  Work was done while the author was at Carnegie Mellon University.  Supported in part by NSF grants CCF-1824303,  CCF-1845146, CCF-2121744 and CMMI-1938909.} \and Benjamin Moseley\thanks{Tepper School of Business, Carnegie Mellon University, Pittsburgh, PA.  \texttt{moseleyb@andrew.cmu.edu}.  Supported in part by  a Google Research Award, an Infor Research Award, a Carnegie Bosch Junior Faculty Chair and NSF grants CCF-1824303,  CCF-1845146, CCF-1733873 and CMMI-1938909.} \and Sergei Vassilvitskii\thanks{Google Research New York, NY.  \texttt{sergeiv@google.com}} }

\begin{document}

\maketitle

\begin{abstract}
The research area of algorithms with predictions has seen recent success showing how to incorporate machine learning into algorithm design to improve performance when the predictions are correct, while retaining worst-case guarantees when they are not.  Most previous work has assumed that the algorithm has access to a single predictor. However, in practice, there are many machine learning methods available, often with incomparable generalization guarantees, making it hard to pick a best method a priori. In this work we consider scenarios where multiple predictors are available to the algorithm and the question is how to best utilize them. 

Ideally, we would like the algorithm's performance to depend on the quality of the {\em best} predictor.  However, utilizing more predictions comes with a cost, since we now have to identify which prediction is the best.  We study the use of multiple predictors for a number of fundamental problems, including matching, load balancing, and non-clairvoyant scheduling, which have been well-studied in the single predictor setting. For each of these problems we introduce new algorithms that take advantage of multiple predictors, and prove bounds on the resulting performance. 
\end{abstract}

\section{Introduction} \label{sec:intro}

An exciting recent line of research attempts to go beyond traditional worst-case analysis of algorithms by equipping algorithms with \emph{machine-learned predictions}.  The hope is that these predictions allow the algorithm to circumvent worst case lower bounds when the predictions are good, and approximately match them otherwise.  The precise definitions and guarantees vary with different settings, but there have been significant successes in applying this framework for many different algorithmic problems, ranging from general online problems to classical graph algorithms (see Section~\ref{sec:related} for a more detailed discussion of related work, and~\cite{MitzenmacherVassilvitskii} for a survey).  In all of these settings it turns out to be possible to define a ``prediction'' where the ``quality'' of the algorithm (competitive ratio, running time, etc.) depends the ``error'' of the prediction.  Moreover, in at least some of these settings, it has been further shown that this prediction is actually learnable with a small number of samples, usually via standard ERM methods~\cite{DinitzILMV21}.  

Previous work has shown the power of accurate predictions, and there are numerous examples showing improved performance in both theory and practice. However, developing accurate predictors remains an art, and a single predictor may not capture all of the subtleties of the instance space. Recently, researchers have turned to working with {\em portfolios of predictors}: instead of training a single model, train multiple models, with the hope that one of them will give good guarantees. 

It is easy to see why the best predictor in a portfolio may be {\em significantly} better than a one-size fits all predictor.  First, many of the modern machine learning methods come with a slew of hyperparameters that require tuning. Learning rate, mini-batch size, optimizer choice, all of these have significant impact on the quality of the final solution. Instead of commiting to a single setting, one can instead try to cover the parameter space, with the hope that some of the predictors will generalize better than others.  Second, problem instances themselves may come from complex distributions, consisting of many latent groups or clusters. A single predictor is forced to perform well on average, whereas multiple predictors can be made to ``specialize'' to each cluster. 

In order to take advantage of the increased accuracy provided by the portfolio approach, we must adapt algorithms with predictions to take advantage of multiple predictions. To capture the gains in performance, the algorithm must perform as if equipped with the best predictor, auto-tuning to use the best one available in the portfolio. However, it is easy to see that there should be a cost as the size of the portfolio grows. In the extreme, one can add every possible prediction to the portfolio, providing no additional information, yet now requiring high performance from the algorithm. Therefore, we must aim to minimize the dependence on the number of predictions in the portfolio. 

We remark that the high level set up may be reminiscent of expert- or bandit-learning literature. However, there is a critical distinction. In expert and bandit learning, we are given a sequence of problem instances, and the goal is to compete (minimize regret) with respect to the best prediction {\em averaged} over the whole sequence. On the other hand, in our setup, we aim to compete with the best predictor on a {\em per-instance} basis.

\paragraph{Previous work on multiple predictions.}
Bhaskara et al. studied an online linear optimization  problem where the learner seeks to minimize the regret, provided access to multiple hints \cite{AdityaOLO20}.
Inspired by the work, Anand et al.  recently studied algorithms with multiple learned predictions in~\cite{AnandGKP22}, proving strong bounds for important online covering problems including online set cover, weighted caching, and online facility location. It was a significant extension of the work \cite{GollapudiP19} which studied the rent-or-buy problem with access to two predictions. However, their techniques and results are limited to online covering problems.  Moreover, they do not discuss the learning aspects at all: they simply assume that they are given $k$ predictions, and their goal is to have competitive ratios that are based on the minimum error of any of the $k$ predictions.  (They actually compete against a stronger dynamic benchmark, but for our purposes this distinction is not important.) 

On the other hand Balcan et al.~\cite{BalcanSV21} look at this problem through a data driven algorithm lens and study the sample complexity and generalization error of working with $k$ (as opposed to 1) parameter settings. The main difference from our work is that they also aim learn a selector, which selects one of the $k$ parameters {\em prior} to beginning to solve the problem instance. In contrast, in this work we make the selection during the course of the algorithm, and sometimes switch back and forth while honing in on the best predictor. 

\subsection{Our Results and Contributions}
In this paper we study three fundamental problems, min-cost perfect matching, online load balancing, and non-clairvoyant scheduling for total completion time, in this new setting. Each of these has seen significant success in the single-prediction model but is not covered by previous multiple-prediction frameworks.  Our results are primarily theoretical, however we have included a preliminary empirical validation of our algorithm for min-cost perfect matching in \submit{the supplementary material}\full{Appendix~\ref{sec:exps}}.

For each of these we develop algorithms whose performance depends on the error of the {\em best} prediction, and explore the effect of the number of predictions, $k$.  Surprisingly, in the case of matching and scheduling we show that using a limited number of predictions is essentially free, and has {\em no} asymptotic impact on the algorithm's performance. For load balancing, on the other hand, we show that the cost of multiple predictions grows {\em logarithmically} with $k$, again implying a tangible benefit of using multiple predictions. We now describe these in more detail.   

\paragraph{Min-Cost Perfect Matching.}  We begin by showcasing our approach with the classical min-cost perfect matching problem in Section~\ref{sec:matching}.  This problem was recently studied by~\cite{DinitzILMV21,ChenSVZ22} to show that it is possible to use learned predictions to improve \emph{running times} of classical optimization problems.  In particular, \cite{DinitzILMV21} showed it is possible to speed up the classical Hungarian algorithm by predicting dual values, and moreover that it is possible to efficiently (PAC-)learn the best duals.  We show that simple modifications of their ideas lead to similar results for multiple predictions.  Interestingly, we show that as long as $k \leq O(\sqrt{n})$, the extra ``cost'' (running time) of using $k$ predictions is negligible compared to the cost of using a single prediction, so we can use up to $\sqrt{n}$ predictions ``for free'' while still getting running time depending on the best of these predictions. Moreover, since in this setting running time is paramount, we go beyond sample complexity to show that it is also computationally efficient to learn the best $k$ predictions.

\paragraph{Online Load Balancing with Restricted Assignments.} We continue in Section~\ref{sec:load_balancing} with the fundamental load balancing problem.  In this problem there are $m$ machines, and $n$ jobs which appear in online fashion.  Each job has a size, and a subset of machines that it can be assigned to. The goal is to minimize the maximum machine load (i.e., the makespan).  This problem has been studied extensively in the traditional scheduling and online algorithms literature, and recently it has also been the subject of significant study given a single prediction~\cite{LattanziLMV,LiX21,LavastidaMRX21}.  In particular, Lattanzi, Lavastida, Moseley, and Vassilvitskii~\cite{LattanziLMV} showed that there exist per machine ``weights'' and an allocation function so that the competitive ratio of the algorithm depends logarithmically on the maximum error of the predictions. We show that one can use $k$ predictions and incur an additional $O(\log k)$ factor in the competitive ratio, while being competitive with the error of the {\em best} prediction.
Additionally, we show that learning the best $k$ predicted weights (in a PAC sense) can be done efficiently.

\paragraph{Non Clairvoyant Scheduling } Finally, in Section~\ref{sec:scheduling} we move to the most technically complex part of this paper.  We study the problem of scheduling $n$ jobs on a single machine, where all jobs are released at time $0$, but where we do not learn the length of a job until it actually completes (the \emph{non-clairvoyant} model).  Our objective is to minimize the sum of completion times. This problem has been studied extensively, both with and without predictions \cite{motwani1994nonclairvoyant,Purohit,Im21spaa,Megow22}.  Most recently, Lindermayr and Megow~\cite{Megow22} suggested that we use an \emph{ordering} as the prediction (as opposed to the more obvious prediction of job sizes), and use the difference between the cost induced by the predicted ordering and the cost induced by the instance-optimal ordering as the notion of ``error''.  In this case, simply following the predicted ordering yields an algorithm with error equal to the prediction error.  

We extend this to the multiple prediction setting, which turns out to be surprisingly challenging. The algorithm of~\cite{Megow22} is quite simple: follow the ordering given by the prediction (and run a 2-competitive algorithm in parallel to obtain a worst-case backstop). But we obviously cannot do this when we are given multiple orderings!  So we must design an algorithm which considers all $k$ predictions to build a schedule that has error comparable to the error of the \emph{best} one.  Slightly more formally, we prove that we can bound the sum of completion times by $(1+\epsilon)\opt$ plus $\text{poly}(1/\epsilon)$ times the error of the best prediction, under the mild assumption that no set of at most $\log \log n$ jobs has a large contribution to $\opt$.

To do this, we first use sampling techniques similar to those of 
\cite{Im21spaa} to estimate the size of the approximately $\epsilon n$'th smallest job without incurring much cost. We then use even more sampling and partial processing to determine for each prediction whether its $\epsilon n$ prefix has many jobs that should appear later (a bad sequence) or has very few jobs that should not be in the prefix (a good sequence).  If all sequences are bad then every prediction has large error, so we can use a round robin schedule and charge the cost to the prediction  error.  Otherwise, we choose one of the good orderings and follow it for its $\epsilon n$ prefix (being careful to handle outliers).  We then recurse on the remaining jobs. 

\subsection{Related Work} \label{sec:related}
As discussed, the most directly related papers are Anand et al.~\cite{AnandGKP22} and Balcan, Sandholm, and Vitercik~\cite{BalcanSV21}; these give the two approaches (multiple predictions and portfolio-based algorithm selection) that are most similar to our setting.  The single prediction version of min-cost bipartite matching was studied in~\cite{DinitzILMV21,ChenSVZ22}, the single prediction version of our load balancing problem was considered by~\cite{LattanziLMV,LiX21,LavastidaMRX21} (and a different though related load balancing problem was considered by~\cite{AhmadianEMP22}), and the single prediction version of our scheduling problem was considered by~\cite{Megow22} with the same prediction that we use (an ordering) and earlier with different predictions by~\cite{Purohit,WeiZ20,Im21spaa}.  Online scheduling with estimates of the true processing times was considered in ~\cite{AzarLT21,AzarLT22}.

More generally, there has been an enormous amount of recent progress on algorithms with predictions.  This is particularly true for online algorithms, where the basic setup was formalized by~\cite{LykourisVassilvitskii} in the context of caching.  For example, the problems considered include caching~\cite{LykourisVassilvitskii, Rohatgi, Panigrahy}, secretary problems~\cite{Secretary, DuttingLPV}, ski rental~\cite{Purohit,AnandGP20, WeiZ20}, and set cover~\cite{BamasMS20}.  There has also been recent work on going beyond traditional online algorithms, including work on running times~\cite{DinitzILMV21,ChenSVZ22}, algorithmic game theory~\cite{AgrawalBGOT,GkatzelisKST,MedinaV17}, and streaming algorithms~\cite{DuWM21,AamandLearned,HsuLearned}.
The learnability of predictions for online algorithms with predictions was considered by~\cite{AnandGKP21}.  They give a novel loss function tailored to their specific online algorithm and prediction, and study the sample complexity of learning a mapping from problem features to a prediction.  While they are only concerned with the sample complexity of the learning problem, we also consider the computational complexity, giving polynomial time $O(1)$-approximate algorithms for the learning problems associated with min-cost matching and online load balancing.

The above is only a small sample of the work on algorithms with predictions.  We refer the interested reader to a recent survey~\cite{MitzenmacherVassilvitskii}, as well as a recently set up website which maintains a list of papers in the area~\cite{ALPSweb}. 

\section{Learnability of \texorpdfstring{$k$}{k} Predictions and Clustering} 
\label{app:pseudo}

In this section we discuss the learnability of $k$ predictions and its connection to $k$-median clustering, which we apply to specific problems in later sections.

\medskip \noindent
\textbf{Learnability and Pseudo-dimension:} 
Consider a problem $\cI$ and let $\cD$ be an unknown distribution over instances of $\cI$.  
We are interested in the learnability of predictions for such problems with respect to an error function.
Suppose that the predictions come from some space $\Theta$ and for a given instance $I\in\cI$ and prediction $\theta \in \Theta$ the error is $\eta(I,\theta)$.  
Given $S$ independent samples $\{I_s\}_{s=1}^S$ from $\cD$, we would like to compute $\hat{\theta} \in \Theta$ such that with probability at least $1-\delta$, we have that 
\begin{equation} \label{eqn:pac_guarantee}
\E_{I \sim \cD}[\eta(I,\hat{\theta})] \leq \min_{\theta \in \Theta} \E_{I \sim \cD}[\eta(I,\theta)] + \epsilon.
\end{equation}
We would like $S$ (the sample complexity) to be polynomial in $\frac{1}{\epsilon},\frac{1}{\delta}$ and other parameters of the problem $\cI$ (e.g. the number of vertices in matching, or the number of jobs and machines in scheduling and load balancing).

The natural algorithm for solving this problem is empirical risk minimization (ERM): take $\hat{\theta} = \arg\min_{\theta \in \Theta} \frac{1}{S} \sum_{s=1}^S \eta(I_s,\theta)$.
The sample complexity of ERM, i.e. how large $S$ should be so that \eqref{eqn:pac_guarantee} holds, can be understood in terms of the pseudo-dimension of the class of functions $\{\eta(\cdot,\theta) \mid \theta \in \Theta\}$.  More generally, the pseudo-dimension can be defined for any class of real valued function on some space $\cX$.

\begin{definition} \cite{pollard2012convergence,anthony2009neural, tim-pseudo}
Let $\mathcal F$ be a class of functions $f: \cX \to \R$.  Let $C = \{x_1,x_2,\ldots,x_S\} \subset \cX$.  We say that that $C$ is \emph{shattered} by $\mathcal F$ if there exist real numbers $r_1,\ldots,r_S$ so that for all $C' \subseteq C$, there is a function $f \in \mathcal F$ such that $f(x_i) \leq r_i \iff x_i \in C'$ for all $i \in [S]$.  The \emph{pseudo-dimension} of $\mathcal F$ is the largest $S$ such that there exists an $C\subseteq X$ with $|C| = S$ that is shattered by $\mathcal F$.
\end{definition}

The pseudo-dimension allows us to give a bound on the number of samples required for uniform convergence, which in turn can be used to show that ERM is sufficient for achieving \eqref{eqn:pac_guarantee}.

\begin{theorem} \label{thm:uniform_convergence} \cite{pollard2012convergence,anthony2009neural,tim-pseudo} 
Let $\cD$ be a distribution over a domain $X$ and $\cF$ be a class of functions $f: \cX \to [0,H]$ with pseudo-dimension $d_{\cF}$.  Consider $S$ independent samples $x_1,x_2,\ldots,x_S$ from $\cD$.  There is a universal constant $c_0$, such that for any $\epsilon > 0 $ and $\delta \in (0,1)$, if 
$S \geq c_0 \left( \frac{H}{\eps}\right)^2 (d_{\cF} + \ln(1/\delta))$
then we have 
\[
\left| \frac{1}{s} \sum_{s=1}^S f(x_i) - \E_{x \sim \cD}[f(x)] \right| \leq \epsilon
\]
for all $f \in \cF$ with probability at least $1-\delta$.
\end{theorem}

We can extend this learning problem to the setting of multiple predictions.  The setup is the same as above, except that now we are interested in outputting $k$ predictions $\hat{\theta}^1,\hat{\theta}^2,\ldots,\hat{\theta}^k$ such that with probability at least $1-\delta$:
\begin{equation} \label{eqn:pac_guarantee_multiple_preds}
\E_{I \sim \cD}\left[\min_{\ell \in [k]}\eta(I,\hat{\theta}^\ell)\right] \leq \min_{\theta^1,\theta^2,\ldots,\theta^k \in \Theta} \E_{I \sim \cD}\left[\min_{\ell \in [k]}\eta(I,\theta^\ell)\right] + \epsilon.
\end{equation}

We can again consider ERM algorithms for this task, and we would like to bound the sample complexity.  We do this by showing that if the pseudo-dimension of the class of functions associated with one prediction is bounded, then it is also bounded for $k$ predictions.  More formally, we want to bound the pseudo-dimension of the class of functions $\{\min_{\ell \in [k]}\eta(\cdot,\theta^\ell) \mid \theta^1,\theta^2,\ldots,\theta^k \in \Theta\}$.  This can be done via the following result combined with Theorem~\ref{thm:uniform_convergence} (assuming that the pseudo-dimension of $\{\eta(\cdot,\theta) \mid \theta \in \Theta\}$ is bounded). 

\begin{theorem} \label{thm:sample_complexity_multiple_predictions_app}
Let $\mathcal F$ be a class of functions $f: \cX \to \R$ with pseudo-dimension $d$ and let $\cF^k := \{F(x) = \min_{\ell \in [k]} f^\ell(x) \mid f^1,f^2,\ldots,f^k \in \cF\}$.  Then the pseudo-dimension of $\cF^k$ is at most $\tilde{O}(dk)$.
\end{theorem}
\begin{proof}
To show this we first relate things back to VC-dimension.

\begin{proposition}
Let $\cF$ be a class of real valued functions on $\cX$.  Define $\cH$ as a class of binary functions on $\cX \times \R$ as $\cH = \{h(x,r) = \sgn(f(x)-r) \mid f \in \cF\}$.  Then the pseudo-dimension of $\cF$ equals the VC dimension of $\cH$.
\end{proposition}

The above proposition follows directly from the definition of pseudo- and VC-dimensions.  The next lemma we need is well known.

\begin{proposition}[Sauer-Shelah Lemma]
If $\cH$ has VC-dimension $d$ and $x_1,\dots,x_m$ is a sample of size $m$, then the number of sets shattered by $\cH$ is at most $O(m^d)$.
\end{proposition}

Let $x_1,\ldots,x_m \in \cX$ and $r_1,\ldots,r_m \in \R$ be given.  We upper bound the number of possible labelings induced by $\cF_k$ on this set.  Note that on this sample the Sauer-Shelah Lemma implies that the number of labelings induced by $\cF$ on this sample is at most $O(m^d)$.  $\cF_k$ allows us to choose $k$ functions from $\cF$ so this increases the number of possible labelings to at most $O(m^{dk})$.  We shatter this set if this bound is greater than $2^m$.  Reorganizing these bounds implies that $m = \tilde{O}(dk)$, which implies the upper bound on the pseudo-dimension of $\cF_k$.
\end{proof}

The associated ERM problem for computing $k$ predictions becomes more interesting.  Recall that in this problem we are given a sample of $S$ instances $I_1,I_2,\ldots,I_S \sim \cD$ and we want to compute $\hat{\theta}^1,\hat{\theta}^2,\ldots,\hat{\theta}^k \in \Theta$ in order to minimize $\frac{1}{S} \sum_{s=1}^S \min_{\ell \in [k]} \eta(I_s,\hat{\theta}^\ell)$.  This can be seen as an instance of the $k$-median clustering problem where we want to cluster the ``points'' $\{I_s\}_{s=1}^S$ by opening $k$ ``facilities'' from the set $\Theta$ and the cost of assigning $I_s$ to $\theta$ is $\eta(I_s,\theta)$.  In general, this problem may be hard to solve or even approximate.  In the case that the costs have some metric structure, then it is known how to compute $O(1)$-approximate solutions  \cite{LiS16}.  For minimum cost matching (Section~\ref{sec:matching} and load balancing (Section~\ref{sec:load_balancing}), we will show that the ERM problem can be seen as a $k$-median problem on an appropriate (pseudo-)metric space.

\medskip \noindent
\textbf{Metrics and Clustering: } 
Recall that $(\cX,d)$ is a metric space if the distance function $d: \cX \times \cX \to \R_+$ satisfies the following properties:
\begin{enumerate}
    \item For all $x,y \in \cX$, $d(x,y) = 0 \iff x=y$
    \item  For all $x,y \in \cX$, $d(x,y) = d(y,x)$
    \item For all $x,y,z \in \cX$, $d(x,z) \leq d(x,y) + d(y,z)$
\end{enumerate}
If we replace the first property with the weaker property that for all $x \in \cX, d(x,x) = 0$, then we call $(\cX,d)$ a pseudo-metric space.

Given a finite set of points $X \subseteq \cX$, the $k$-median clustering problem is to choose a subset $C \subseteq \cX$ of $k$ centers to minimize the total distance of each point in $X$ to its closest center in $C$.  In notation, the goal of $k$-median clustering is to solve

\begin{equation}
    \min_{C \subseteq \cX, |C| = k} \sum_{x \in X} \min_{c \in C} d(x,c)
\end{equation}

In our settings it will often be challenging to optimize $C$ over all of $\cX$, so at an $O(1)$-factor loss to the objective we can instead optimize $C$ over $X$.  Formally, we have the following standard lemma.

\begin{lemma} \label{lem:clustering_approx}
Let $(\cX,d)$ be a pseudo-metric space and let $X$ be a finite subset of $\cX$.  Then for all $k > 0$ we have
\[
\min_{C \subseteq X, |C| = k} \sum_{x \in X} \min_{c \in C} d(x,c) \leq 2 \cdot \min_{C \subseteq \cX, |C| = k} \sum_{x \in X} \min_{c \in C} d(x,c).
\]
\end{lemma}
\begin{proof}
Let $C^* \subseteq \cX$ be an optimal solution to the problem on the right hand side of the inequality, and let its cost be $\opt$.  We consider a mapping $\phi:C^* \to X$, which gives us a solution to the problem on the left hand side by taking $C = \{\phi(c) \mid c \in C^*\}$\footnote{In the case that $|C| < k$, adding arbitrary points from $X$ to $C$ so that $|C|=k$ can only decrease the cost from just using $C$.}.  For $c \in C^*$, define $\phi(c) = \arg\min_{x \in X}d(x,c)$  We will argue that the cost of $C$ is at most $2\opt$.  Let $x \in X$ and let $c^* = \arg\min_{c \in C^*}d(x,c)$ be its closest center in $C^*$, then we have 
\[
\begin{split}
    \min_{c \in C} d(x,c) \leq d(x,\phi(c^*)) \leq d(x,c^*) + d(c^*,\phi(c^*)) \leq 2d(x,c^*) = 2\min_{c \in C^*} d(x,c)
\end{split}
\]
The second to last inequality follows from the triangle inequality and the last follows from the definition of $\phi$.  Now summing over all $x \in X$ yields that the cost of using $C$ is at most $2\opt$.
\end{proof}

\section{Minimum Cost Bipartite Matching with Predicted Duals} \label{sec:matching}

In this section we study the minimum cost bipartite matching problem with multiple predictions.  The case of a single prediction has been considered recently~\cite{DinitzILMV21,ChenSVZ22}, where they used dual values as a prediction and showed that the classical Hungarian algorithm could be sped up by using appropriately learned dual values.  Our goal in this section is to extend these results to multiple predictions, i.e., multiple duals.  In particular, in Section~\ref{sec:matching-using-predictions} we show that we can use $k$ duals and get running time comparable to the time we would have spent if we used the single best of them in the algorithm of~\cite{DinitzILMV21}, with no asymptotic loss if $k$ is at most $O(\sqrt{n})$.  Then in Section~\ref{sec:matching-learning} we show that $k$ predictions can be learned with not too many more samples (or running time) than learning a single prediction.  

\subsection{Problem Definition and Predicted Dual Variables} 
In the minimum cost bipartite matching problem we are given a bipartite graph $G = (V,E)$ with $n = |V|$ vertices and $m = |E|$ edges, with edge costs $c \in \Z^E$.  The  objective is to output a perfect matching $M \subseteq E$ which minimizes the cost $c(M) := \sum_{e \in E}c_e$.  This problem is exactly captured by the following primal and dual linear programming formulations.

\full{
\begin{equation} \tag{MWPM-P} \label{eqn:mwpm_lp}
    \begin{array}{ccc}
        \min  & \displaystyle \sum_{e \in E}  c_ex_e &  \\
         & \displaystyle \sum_{e \in N(i)} x_e = 1 & \forall i \in V\\
         & x_{e} \geq 0 & \forall e \in E
    \end{array}
\end{equation}
\begin{equation} \tag{MWPM-D} \label{eqn:mwpm_dual}
    \begin{array}{ccc}
         \max & \displaystyle \sum_{i \in V} y_i &  \\
         & \displaystyle y_i + y_j \leq c_e & \forall e = ij \in E \\
    \end{array}
\end{equation}
}

\submit{
 \begin{wrapfigure}{R}{0.375\textwidth}
    \vspace{-3mm}
    \begin{minipage}{.325\textwidth}
\begin{equation*}  \label{eqn:mwpm_lp}
    \begin{array}{ccc}
        \min  & \displaystyle \sum_{e \in E}  c_ex_e & \textrm{(MWPM-P)} \\
         & \displaystyle \sum_{e \in N(i)} x_e = 1 & \forall i \in V\\
         & x_{e} \geq 0 & \forall e \in E
    \end{array}
\end{equation*}

\begin{equation*}  \label{eqn:mwpm_dual} 
    \begin{array}{ccc}
         \max & \displaystyle \sum_{i \in V} y_i & \textrm{(MWPM-D)} \\
         & \displaystyle y_i + y_j \leq c_e & \forall e = ij \in E \\
    \end{array}
\end{equation*}
\end{minipage}
\vspace{-2mm}
\end{wrapfigure}
}

Dinitz et al.~\cite{DinitzILMV21} studied initializing the Hungarian algorithm with a prediction $\yhat$ of the optimal dual solution $y^*$.  They propose an algorithm which operates in two steps\full{ (see Algorithm~\ref{alg:learned_duals} for pseudo-code)}. First, the predicted dual solution $\yhat$ may not be feasible, so they give an $O(n + m)$ time algorithm which recovers feasibility (which we refer to as Make-Feasible).  Second, the now-feasible dual solution is used in a primal-dual algorithm such as the Hungarian algorithm (which we refer to as Primal-Dual) and they show that the running time depends on the $\ell_1$ error in the predicted solution.  In addition to this they show that learning a good initial dual solution is computationally efficient with low sample complexity.  More formally, they proved the following theorems.

\full{
\begin{theorem}[\citet{DinitzILMV21}]
Let $(G,c)$ be an instance of minimum cost bipartite matching and $\yhat$ be a prediction of an optimal dual solution $y^*$.  Algorithm~\ref{alg:learned_duals} returns an optimal solution and runs in time $O(m\sqrt{n} \cdot \|y^* -\yhat\|_1)$.  Moreover, the Make-Feasible step of Algorithm~\ref{alg:learned_duals} runs in $O(n+m)$ time.
\end{theorem}
}

\submit{
\begin{theorem}[\citet{DinitzILMV21}]
Let $(G,c)$ be an instance of minimum cost bipartite matching and $\yhat$ be a prediction of an optimal dual solution $y^*$.  There exists an algorithm which returns an optimal solution and runs in time $O(m\sqrt{n} \cdot \|y^* -\yhat\|_1)$.  
\end{theorem}
}

\begin{theorem}[\citet{DinitzILMV21}] \label{thm:primal_dual_learning}
Let $\cD$ be an unknown distribution over instances $(G,c)$ on $n$ vertices and let $y^*(G,c)$ be an optimal dual solution for the given instance.  Given $S$ independent samples from $\cD$, there is a polynomial time algorithm that outputs a solution $\yhat$ such that
\[\E_{(G,c)\sim \cD}\left[ \|y^*(G,c) - \yhat\|_1 \right] \leq \min_{y}\E_{(G,c)\sim \cD}\left[ \|y^*(G,c) -y \|_1 \right] +\epsilon\]
with probability $1-\delta$ where $S = \poly(n,\frac{1}{\epsilon},\frac{1}{\delta})$.
\end{theorem}

\full{
\begin{algorithm}[H]
\caption{Minimum cost matching with a predicted dual solution \label{alg:learned_duals}}
\begin{algorithmic}[1]
\Procedure{PredictedPrimal-Dual}{$G,c,\yhat$}
\State $y \gets $MakeFeasible$(G,c,\yhat)$
\State $M \gets $Primal-Dual$(G,c,y)$
\State Return $M$
\EndProcedure
\end{algorithmic}
\end{algorithm}
}

\subsection{Using \texorpdfstring{$k$}{k} Predicted Dual Solutions Efficiently} \label{sec:matching-using-predictions}

Given $k$ predicted dual solutions $\yhat^1,\yhat^2,\ldots,\yhat^k$, we would like to efficiently determine which solution has the minimum error for the given problem instance.  Note that the predicted solutions may still be infeasible and that we do not know the target optimal dual solution $y^*$.  We propose the following simple algorithm which takes as input $k$ predicted solutions and whose running time depends only on the $\ell_1$ error of the \emph{best} predicted solution.  First, we make each predicted solution feasible, just as before.  Next, we select the (now-feasible) dual solution with highest dual objective value and proceed running the primal-dual algorithm with only that solution.  See Algorithm~\ref{alg:multiple_learned_duals} for pseudo-code.

\submit{
\begin{wrapfigure}{R}{0.7\textwidth}
\vspace{-3mm}
\begin{minipage}{.7\textwidth}
 \begin{algorithm}[H]
\caption{Minimum cost matching with $k$ predicted dual solutions \label{alg:multiple_learned_duals}}
\begin{algorithmic}[1]
\Procedure{$k$-PredictedPrimal-Dual}{$G,c,\yhat^1,\yhat^2,\ldots,\yhat^k$}
\For{$\ell \in [k]$}
\State $y^\ell \gets $MakeFeasible$(G,c,\yhat^\ell)$
\EndFor
\State $\ell' \gets \arg \max_{\ell \in [k]} \sum_{i \in V} y^\ell_i$
\State $M \gets $Primal-Dual$(G,c,y^{\ell'})$
\State Return $M$
\EndProcedure
\end{algorithmic}
\end{algorithm} 
\end{minipage}
\vspace{-3mm}
\end{wrapfigure}
}

\full{ 
\begin{algorithm}[H]
\caption{Minimum cost matching with $k$ predicted dual solutions \label{alg:multiple_learned_duals}}
\begin{algorithmic}[1]
\Procedure{$k$-PredictedPrimal-Dual}{$G,c,\yhat^1,\yhat^2,\ldots,\yhat^k$}
\For{$\ell \in [k]$}
\State $y^\ell \gets $MakeFeasible$(G,c,\yhat^\ell)$
\EndFor
\State $\ell' \gets \arg \max_{\ell \in [k]} \sum_{i \in V} y^\ell_i$
\State $M \gets $Primal-Dual$(G,c,y^{\ell'})$
\State Return $M$
\EndProcedure
\end{algorithmic}
\end{algorithm} 
}

We have the following result concerning Algorithm~\ref{alg:multiple_learned_duals}.  To interpret this result, note that the cost for increasing the number of predictions is $O(k(n+m))$, which will be dominated by the $m\sqrt{n}$ term we pay for running the Hungarian algorithm unless $k$ is extremely large (certainly larger than $\sqrt{n}$) or there is a prediction with $0$ error (which is highly unlikely).  Hence we can reap the benefit of a large number of predictions ``for free''.

\begin{theorem}
Let $(G,c)$ be a minimum cost bipartite matching instance and let $\yhat^1,\yhat^2,\ldots,\yhat^k$ be predicted dual solutions.  Algorithm~\ref{alg:multiple_learned_duals} returns an optimal solution and runs in time $O(k(n+m) + m\sqrt{n} \cdot \min_{\ell \in [k]} \|y^* - \yhat^\ell\|_1)$.
\end{theorem}
\full{
\begin{proof}
The correctness of the algorithm (i.e., returning an optimal solution) follows from the correctness of Algorithm~\ref{alg:learned_duals}.  For the running time, we clearly spend $O(k(n+m))$ time making each predicted solution feasible, thus we just need to show the validity of latter term in the running time.  Let $\ell'=\arg \max_{\ell \in [k]} \sum_{i \in V} y^\ell_i$ be the solution chosen in line 5 of Algorithm~\ref{alg:multiple_learned_duals} and let $\ell^* = \arg\min_{\ell \in [k]} \|y^*- \yhat^\ell\|_1$ be the solution with minimum error.  Recall that for each $\ell \in [k]$, $y^\ell$ is the resulting \emph{feasible} dual solution from calling Make-Feasible on $\yhat^\ell$.  By the analysis from~\cite{DinitzILMV21}, we have that $\|y^* - y^{\ell^*}\|_1 \leq 3 \|y^* - \yhat^{\ell^*}\|_1$, so it suffices to show that the number of primal-dual iterations will be bounded by $\|y^* - y^{\ell^*}\|_1$.  By our choice of $\ell'$, we have $\sum_i y_i^{\ell'} \geq \sum_i y_i^{\ell^*}$, therefore we have that $\sum_i (y^*_i - y^{\ell'}) \leq \sum_i(y^*_i - y_i^{\ell^*}) \leq \|y^* - y^{\ell^*}\|_1$.  From the analysis in ~\cite{DinitzILMV21}, we have that the number of primal-dual iterations will be at most $\sum_i (y^*_i - y^{\ell'})$, completing the proof.
\end{proof}
}
\submit{
We defer the proof to the supplementary material.  But correctness is essentially direct from~\cite{DinitzILMV21}, and the running time requires just a simple modification of the analysis of~\cite{DinitzILMV21}.
}

\subsection{Learning \texorpdfstring{$k$}{k} Predicted Dual Solutions} \label{sec:matching-learning}

Next we extend Theorem~\ref{thm:primal_dual_learning} to the setting where we output $k$ predictions.  Let $\cD$ be a distribution over problem instances $(G,c)$ on $n$ vertices.  We show that we can find the best set of $k$ predictions.  More formally, we prove the following theorem. 

\begin{theorem}
Let $\cD$ be an unknown distribution over instances $(G,c)$ on $n$ vertices and let $y^*(G,c)$ be an optimal dual solution for the given instance.  Given $S$ independent samples from $\cD$, there is a polynomial time algorithm that outputs $k$ solutions $\yhat^1,\yhat^2,\ldots,\yhat^k$ such that
\[\E_{(G,c)\sim \cD}\left[ \min_{\ell \in [k]}\|y^*(G,c) - \yhat^\ell\|_1 \right] \leq O(1)\cdot\min_{y^1,y^2,\ldots,y^k} \E_{(G,c) \sim \cD} \left[ \min_{\ell \in [k]} \|y^*(G,c) - y^{\ell}\|_1 \right] +\epsilon\]
with probability $1-\delta$ where $S = \poly(n,k,\frac{1}{\epsilon},\frac{1}{\delta})$.
\end{theorem}

\submit{The proof of this theorem can be found in the supplementary material, but it is straightforward.  The sample complexity is due to combining Theorem~\ref{thm:sample_complexity_multiple_predictions} with Theorem~\ref{thm:primal_dual_learning} (or more precisely, with the pseudo-dimension bound which implies Theorem~\ref{thm:primal_dual_learning}).  The $O(1)$-approximation factor and polynomial running time is from the observation that the ERM problem in this setting is just an instance of the $k$-median clustering problem.  }

\full{
\begin{proof}
By Theorem 7 in ~\cite{DinitzILMV21} and Theorems~\ref{thm:uniform_convergence} and~\ref{thm:sample_complexity_multiple_predictions_app}, we get the polynomial sample complexity.  Thus we just need to give a polynomial time ERM algorithm to complete the proof.  Let $\{(G^s,c^s\}_{s=1}^S$ be the set of sampled instances.  We start by computing $z^s = y^*(G^s,c^s) \in \Z^V$ for each $s \in [S]$.  Consider the ERM problem where we replace the expectation by a sample average:
\[
\min_{y^1,y^2,\ldots,y^k} \frac{1}{S} \sum_{s=1}^S\min_{\ell \in [k]} \|z^s - y^{\ell}\|_1
\]
This can be seen as a $k$-median clustering problem where each predicted solution $y^\ell$ is a cluster center and distances are given by the $\ell_1$ norm.  Thus we can find an $O(1)$-approximate solution $\yhat^1,\yhat^2,\ldots,\yhat^k$ to this problem which is of polynomial size in polynomial time (for example by applying the algorithm due to~\cite{LiS16}).
\end{proof}
}

\section{Online Load Balancing with Predicted Machine Weights} \label{sec:load_balancing}

We now apply our framework to online load balancing with restricted assignments.  In particular, we consider proportional weights, which have been considered in prior work~\cite{LattanziLMV,LiX21,LavastidaMRX21}.   Informally, we show in Section~\ref{sec:load_balancing_combine} that if $\beta$ is the cost of the \emph{best} of the $k$ predictions, then even without knowing a priori which prediction is best, we get cost of $O(\beta \log k)$.  Then in Section~\ref{sec:load-balancing-learning} we show that it does not take many samples to actually learn the best $k$ predictions.

\subsection{Problem Definition and Proportional Weights}

In online load balancing with restricted assignments there is a sequence of $n$ jobs which must be assigned to $m$ machines in an online fashion.  Upon seeing job $j$, the online algorithm observes its size $p_j > 0$ and a neighborhood $N(j) \subseteq [m]$ of \emph{feasible} machines.  The algorithm must then choose some feasible machine $i \in N(j)$ to irrevocably assign the job to before seeing any more jobs in the sequence.  We also consider fractional assignments, i.e. vectors belonging to the set $X = \{x \in \R_+^{m \times n} \mid \forall j \in [n], \sum_i x_{ij} =1, \text{ and } x_{ij} = 0 \iff i \notin N(j)\}$.

Prior work studied the application of proportional weights\cite{AgrawalZM18,LavastidaMRX21,LattanziLMV,LiX21}.   Intuitively, a prediction in this setting is a weighting of \emph{machines}, which then implies an online assignment, which is shown to be near-optimal.  Slightly more formally, suppose that we are given weights $w_i$ for each machine $i$.  Then each job $j$ is \emph{fractionally} assigned to machine $i$ to a fractional amount of $\frac{w_i}{ \sum_{i'\in N(j)} w_{i'}}$.  Notice that given weights, this also gives an online assignment.  It is known that there exist weights for any instance where the fractional solution has a near optimal makespan, even though there are only $m$ ``degree of freedom'' in the weights compared to $mn$ in an assignment. That is, for all machines $i$,  $ \sum_{j\in [n]} p_j \cdot \frac{w_i}{ \sum_{i'\in N(j)} w_{i'}} $ is at most a $(1+\epsilon)$ factor larger than the optimal makespan for any constant $\epsilon >0$ \cite{AgrawalZM18,LattanziLMV}.  

Let $w^*$ be a set of near optimal weights for a given instance. Lattanzi et al.~\cite{LattanziLMV} showed the following theorem:
\begin{theorem}
\label{thm:LLMV}
Given predicted weights $\what$, there is an online fractional algorithm which has makespan $O(\log (\eta(\what, w^*) \opt)$, where $\eta(\what, w^*) := \max_{i \in [m]} \max(\frac{\what_i}{w^*_i},\frac{w^*_i}{\what_i})$ to be the error in the prediction.  
\end{theorem}

Moreover, this fractional assignment can be converted online to an integral assignment while losing only an $O(\log \log m)$ factor in the makespan~\cite{LattanziLMV,LiX21}.  Thus, we focus on constructing fractional assignments that are competitive with the best prediction in hindsight.

\subsection{Combining Fractional Solutions Online} \label{sec:load_balancing_combine}

Given $k$ different predicted weight vectors $\what^1,\what^2,\ldots,\what^k$, we want to give an algorithm which is competitive against the \emph{minimum} error among the predicted weights,  i.e. we want the competitiveness to depend upon $\eta_{\min} := \min_{\ell \in [k]} \eta(\what^\ell,w^*)$.

The challenge is that we do not know up front which $\ell \in [k]$ will yield the smallest error, but instead learn this in hindsight.  For each $\ell \in [k]$, let $x^\ell$, be the resulting fractional assignment from applying the fractional online algorithm due to~\cite{LattanziLMV} with weights $\what^\ell$.  This fractional assignment is revealed one job at a time.  

We give an algorithm which is $O(\log k)$-competitive against any collection of $k$ fractional assignments which are revealed online.  Moreover, our result applies to the unrelated machines setting, in which each job has a collection of machine-dependent sizes $\{p_{ij}\}_{i \in [m]}$.  The algorithm is based on the doubling trick and is similar to results in \cite{AzarBM93} which apply to metrical task systems.  Let $\beta := \min_{\ell \in [k]} \max_i \sum_j p_{ij}x^\ell_{ij}$ be the best fractional makespan in hindsight. As in previous work, our algorithm is assumed to know $\beta$, an assumption that can be removed~\cite{LattanziLMV}.  At a high level, our algorithm maintains a set $A \subseteq [k]$ of solutions which are good with respect to the current value of $\beta$, averaging among these.  See Algorithm~\ref{alg:load_balancing_combine} for a detailed description.  We have the following theorem.

\begin{theorem} \label{thm:load_balancing_combine}
Let $x^1,x^2,\ldots,x^k$ be fractional assignments which are revealed online.  If Algorithm~\ref{alg:load_balancing_combine} is run with $\beta := \min_{\ell \in [k]} \max_i \sum_j p_{ij}x^\ell_{ij}$, then it yields a solution of cost $O(\log k)\cdot \beta$ and never reaches the fail state (line 7 in Algorithm~\ref{alg:load_balancing_combine}).
\end{theorem}

Let $\beta_{\ell} = \max_i \sum_j p_{ij}x^{\ell}_{ij}$ and $\opt$\ be the optimal makespan. Theorem ~\ref{thm:LLMV} shows that $\beta_{\ell} \leq O(\log \eta_{\ell}) \opt$. The following corollary is then immediate:
\begin{corollary}
Let $w^1, w^2, \ldots, w^k$ be the predicted weights with errors $\eta^1, \eta^2, \ldots, \eta^k$. Then Algorithm~\ref{alg:load_balancing_combine} returns a fractional assignment with makespan at most $\opt \cdot O(\log k) \cdot \min_{\ell \in [k]} \log (\eta^{\ell})$. 
\end{corollary}

\begin{algorithm}[H]
\caption{Algorithm for combining fractional solutions online for load balancing. \label{alg:load_balancing_combine}}
\begin{algorithmic}[1]
\Procedure{Combine-LoadBalancing}{$\beta$}
\State $A \gets [k]$ \Comment{Initially all solutions are good}
\For{each job $j$}
\State Receive the assignments $x^1,x^2,\ldots,x^k$
\State $A(j,\beta) \gets \{ \ell \in A \mid \forall i \in [m], x^\ell_{ij} > 0 \implies p_{ij}x^\ell_{ij} \leq \beta\}$
\If{$A = \emptyset$ or $A(j,\beta) = \emptyset$}
\State Return ``Fail''
\EndIf
\State $\forall i \in [m]$, $x_{ij} \gets \frac{1}{|A(j,\beta)|} \sum_{\ell \in A(j,\beta)} x^\ell_{ij}$
\State $B \gets \{ \ell \in A \mid \max_{i \in [m]} \sum_{j' \leq j} p_{ij'} x^\ell_{ij'} > \beta \}$ \Comment{Bad solutions w.r.t. $\beta$}
\State $A \gets A \setminus B$
\EndFor
\EndProcedure
\end{algorithmic}
\end{algorithm}

\submit{
We defer the proof of Theorem~\ref{thm:load_balancing_combine} to the Supplementary material.
}

\full{
The analysis relies on the following decomposition of machine $i$'s load.  Note that in our algorithm we assign a weight $\alpha^\ell_j \in [0,1]$ to solution $\ell$ when computing the fractional assignment for job $j$.  That is, we take $x_{ij}= \sum_{\ell \in A} \alpha^\ell_j x^\ell_{ij}$ where $\alpha^\ell_j = 1/|A(j,\beta)|$ if $\ell \in A(j,\beta)$ and 0 otherwise.  Then the decomposition of machine $i$'s load $L_i$ is
\begin{equation} \label{eqn:load_balancing_decomp}
L_i = \sum_j p_{ij} x_{ij} = \sum_j p_{ij} \sum_{\ell \in [k]} \alpha^\ell_j x^\ell_{ij} = \sum_{\ell \in [k]} \left( \sum_j \alpha^\ell_jp_{ij}x^\ell_{ij} \right).
\end{equation}

Now we define $C^\ell_i:= \sum_j \alpha^\ell_jp_{ij}x^\ell_{ij}$ to be the contribution of solution $\ell$ to machine $i$'s load. Without loss of generality suppose that the order in which solutions are removed from $S$ in Algorithm~\ref{alg:load_balancing_combine} is $1,2,\ldots,k$.

\begin{lemma} \label{lem:load_balancing_bound}
For all $i\in [m]$ and each $\ell \in [k]$, we have that $C^\ell_i \leq \frac{2\beta}{k-\ell+1}$.
\end{lemma}
\begin{proof}
Let $j$ be the job in which $\ell$ is removed from $A$ by our algorithm.  Before this, for all $j' < j$ we had that $\ell$'s fractional makespan was at most $\beta$.  Thus we can write $C_i^\ell$ as:
\begin{align*}
    C^\ell_i & = \sum_j \alpha^\ell_jp_{ij}x^\ell_{ij}  =\sum_{j' < j}\alpha^\ell_{j'}p_{ij'}x^\ell_{ij'}+ \alpha^\ell_jp_{ij}x^\ell_{ij} + \sum_{j' > j}\alpha^\ell_{j'}p_{ij'}x^\ell_{ij'} 
\end{align*}
For the first set of terms, we have that $\alpha^\ell_{j'} \leq \frac{1}{k-\ell+1}$ since $\ell$ has not yet been removed from $A$.  Thus these terms can be bounded above by $\frac{\beta}{k-\ell+1}$, since $\ell$'s fractional makespan was bounded above by $\beta$ before job $j$.  For the middle term, we use a similar observation for $\alpha^\ell_j$ but also apply the definition of $S(j,\beta)$ to conclude that $\alpha^\ell_j >0 \implies p_{ij} x^\ell_{ij} \leq \beta$.  Thus we get a bound of $\frac{\beta}{k-\ell+1}$ for the middle term.  For the final set of terms, we have $\alpha^\ell_{j'} = 0$ for all $j' > j$ since we have removed $\ell$ from $A$ at this point, and so these contribute nothing.  Combining these bounds gives the lemma.
\end{proof}

\begin{proof}[Proof of Theorem~\ref{thm:load_balancing_combine}]
By \eqref{eqn:load_balancing_decomp}, we have that the load of machine $i$ is at most $\sum_{\ell \in [k]} C^\ell_i$.  Applying Lemma~\ref{lem:load_balancing_bound} we have $C^\ell_i \leq \frac{2\beta}{k-\ell+1}$.  Thus machine $i$'s load is at most $\sum_{\ell \in [k]} \frac{2\beta}{k-\ell+1} = 2H_k\beta = O(\log k)\cdot \beta$, where $H_k = \sum_{\ell=1}^k \frac{1}{k}$ is the $k$'th harmonic number.  Now if we run the algorithm with $\beta = \min_{\ell \in [k]} \max_i \sum_j p_{ij}x^\ell_{ij}$, then there is some solution $\ell^* \in [k]$ which has a fractional makespan of $\beta$, so it never gets removed from $A$ or $A(j,\beta)$ in Algorithm~\ref{alg:load_balancing_combine}.  In this case Algorithm~\ref{alg:load_balancing_combine} never fails, completing the proof of the theorem.
\end{proof}

}

\subsection{Learning \texorpdfstring{$k$}{k} Predicted Weight Vectors} \label{sec:load-balancing-learning}

We now turn to the question of showing how to learn $k$ different predicted weight vectors $\what^1,\what^2,\ldots,\what^k$.  Recall that there is an unknown distribution $\cD$ over sets of $n$ jobs from which we receive independent samples $J_1,J_2,\ldots,J_S$.  Our goal is to show that we can efficiently learn (in terms of sample complexity) $k$ predicted weight vectors to minimize the expected minimum error. Let $w^*(J)$ be the correct weight vector for instance $J$ and let $\eta(w,w') = \max_{i \in [m]} \max (\frac{w_i}{w'_i}, \frac{w'_i}{w_i})$ be the error between a pair of weight vectors.  We have the following result.

\begin{theorem} \label{thm:weights_erm}
Let $\cD$ be an unknown distribution over restricted assignment instances on $n$ jobs and let $w^*(J)$ be a set of good weights for instance $J$.  Given $S$ independent samples from $\cD$, there is a polynomial time algorithm that outputs $k$ weight vectors $\what^1,\what^2,\ldots,\what^k$ such that  $\E_{J \sim \cD} \left[ \min_{\ell \in [k]} \log(\eta(\what^\ell, w^*(J)) \right] \leq O(1) \cdot \min_{w^1,w^2,\ldots,w^k} \E\left[ \min_{\ell \in [k]}\log(\eta(w^\ell, w^*(J)) \right] + \epsilon$ 
with probability $1-\delta$, where $S = \poly(m,k,\frac{1}{\epsilon},\frac{1}{\delta})$
\end{theorem}

\submit{
The proof of Theorem~\ref{thm:weights_erm} is deferred to the Supplementary material, but we note that to get a polynomial time algorithm we carry out an interesting reduction to $k$-median clustering.  Namely, we show that the function $d(w,w') := \log(\eta(w,w'))$ satisfies the triangle inequality and thus forms a metric space.
}

\full{

Prior work~\cite{LavastidaMRX21} has observed that we can take the weights to be from the set $\cW(R) = \{ w \in \R_+^m \mid \forall i \in [m], \exists \alpha \in [R], \text{ s.t. } w_i = (1+\epsilon/m)^\alpha\}$.  Moreover, it suffices to take $R =\Theta(m^2\log(m))$ in order to guarantee that for any instance there exists some set of weights in $\cW(R)$ such that the associated fractional assignment yields an $O(1)$-approximate solution.  Since this set is finite, with only $m^{O(m^2\log m)}$ members, it follows that the pseudo-dimension of any class of functions parameterized by the weights is bounded by $\log(|\cW(R)|) = O(m^2\log m)$, and so we get polynomial sample complexity.  Thus in order to prove Theorem~\ref{thm:weights_erm}, it suffices to give a polynomial time ERM algorithm for this problem.  As hinted at in the statement of Theorem~\ref{thm:weights_erm}, we will be working with the logarithms of the errors.  Working in this space allows us to carry out a reduction to $k$-median clustering.

For any pair of weight vectors $w,w' \in \R_+^m$, recall that we define the error between them to be $\eta(w,w') = \max_{i \in [m]} \max (\frac{w_i}{w'_i}, \frac{w'_i}{w_i})$.  Note that $\eta(w,w') \geq 1$ with equality if and only if $w = w'$ and that $\eta(w,w') = \eta(w',w)$.  The main lemma is that defining $d(w,w') := \log(\eta(w,w'))$ satisfies the triangle inequality, and thus forms a metric on $\R_+^m$ due to the aforementioned observations.

\begin{lemma}
Let $\eta:\R_+^m \times \R_+^m$ and $d: \R_+^m \times \R_+^m$ be defined as above.  Then $(\R_+^m,d)$ forms a metric space.
\end{lemma}
\begin{proof}
It is easy to see that $d(w,w') \geq 0$ with equality if and only if $w=w'$ and that $d(w,w') = d(w',w)$.  Thus we just need to show that the triangle inequality holds, i.e. that for all $w,w',w''$ we have $d(w,w'') \leq d(w,w') + d(w',w'')$.  This will hold as a result of the following claim.  For all $w,w',w''$. we have:
\begin{equation}
    \eta(w,w'') \leq \eta(w,w') \cdot \eta(w',w'').
\end{equation}
Now the triangle inequality follows since 
\[
d(w,w') = \log (\eta(w,w'')) \leq \log (\eta(w,w')) + \log (\eta(w',w'')) = d(w,w') + d(w',w''). 
\]
To prove the claim we have the following:
\begin{align*}
    \eta(w,w'') & = \max_i \left\{ \max \left( \frac{w_i}{w''_i} , \frac{w''_i}{w_i} \right) \right\} \\
                & = \max_i \left\{ \max \left( \frac{w_i}{w'_i} \frac{w'_i}{w''_i} , \frac{w''_i}{w'_i} \frac{w'_i}{w_i} \right) \right\} \\
                & \leq \max_i \left\{ \max \left( \frac{w_i}{w'_i}, \frac{w'_i}{w''_i} \right) \cdot \max\left(\frac{w''_i}{w'_i} ,\frac{w'_i}{w_i} \right) \right\} \\
                & \leq \max_i \left\{ \max \left( \frac{w_i}{w'_i}, \frac{w'_i}{w''_i} \right)\right\} \cdot \max_i \left\{  \max \left(\frac{w''_i}{w'_i} ,\frac{w'_i}{w_i} \right)\right\} \\
                & = \eta(w,w') \cdot \eta(w',w'').
\end{align*}
The two inequalities above follow from the next two claims below.
\end{proof}
\begin{claim}
For all $a,b,c > 0$ we have $\max(\frac{a}{c}, \frac{c}{a}) \leq \max(\frac{a}{b}, \frac{b}{a}) \cdot \max(\frac{b}{c}, \frac{c}{b})$
\end{claim}
\begin{proof}
Without loss of generality, we may assume that $a \geq c$.  Now we have several cases depending on the value of $b$.  For the first case, lets consider when $a \geq b \geq c > 0$.  In this case, the left hand side evaluates to $\frac{a}{c}$ while the right hand side also evaluates to $\frac{a}{c}$, so the inequality is valid in this case.

For the next case, consider when $b \geq a \geq c > 0$.  In this case, the left hand side is still $\frac{a}{c}$, while the right hand side evaluates to $\frac{b^2}{ac} \geq \frac{ab}{ac} \geq \frac{a}{c}$.  Thus the inequality is valid.

In the final case, we have $a \geq c \geq b > 0.$  The left hand side is $\frac{a}{c}$, while the right hand side evaluates to $\frac{ac}{b^2} \geq \frac{ac}{c^2} = \frac{a}{c}$, completing the proof.
\end{proof}
\begin{claim}
Let $u,v \in \R_+^m$, then we have $\max_i (u_iv_i) \leq (\max_i u_i) (\max_i v_i)$
\end{claim}
\begin{proof}
Suppose for contradiction that this isn't the case, i.e. $\max_i (u_iv_i) > (\max_i u_i) (\max_i v_i)$.  Now let $i^* = \arg\max_i (u_iv_i)$.  Then we have
\[
u_{i^*}v_{i^*} > (\max_i u_i) (\max_i v_i) \geq u_{i^*}v_{i^*}
\]
which is the desired contradiction.
\end{proof}

\begin{proof}[Proof of Theorem~\ref{thm:weights_erm}]
The sample complexity follows from the discussion above and 
\submit{Theorem~\ref{thm:sample_complexity_multiple_predictions}}
\full{Theorem~\ref{thm:sample_complexity_multiple_predictions_app}}
which shows that the pseudo-dimension of the error function is at most $O(m^2\log(m))$.  Given $S$ samples $J_1,J_2,\ldots,J_S$ from $\cD$, we want to solve the corresponding ERM instance.  To do this we set up the following $k$-median instance to compute the predicted weights $\what^1,\what^2,\ldots,\what^k$.  First we compute $w^s = w^*(J_s)$ for each $s \in [S]$, then we set the distance between $w^s$ and $w^{s'}$ to be the distance function $d(w^s,w^{s'})$ defined above.  At a loss of a factor of $2$, we can take each $\what^\ell$ to be in $\{w^s\}_{s=1}^S$ by Lemma~\ref{lem:clustering_approx}.  Thus we can apply an $O(1)$-approximate $k$-median algorithm (e.g. the one due to~\cite{LiS16}) to get the predicted weight vectors $\what^1,\what^2,\ldots,\what^k$ in polynomial time.
\end{proof}

}

\section{Scheduling with Predicted Permutations} \label{sec:scheduling}

In this problem there are $n$ jobs, indexed by $1, 2, \ldots, n$, to be scheduled on a single machine. We assume that they are all available at time 0. 
Job $j$ has size $p_j$ and needs to get processed for $p_j$ time units to complete. If all job sizes are known a priori, Shortest Job First (or equivalently Shortest Remaining Time First), which processes jobs in non-decreasing order of their size, is known to be optimal for minimizing total completion time. 
We assume that the true value of $p_j$ is unknown and is revealed only when the job completes, i.e. the \emph{non-clairvoyant} setting. In the non-clairvoyant setting, it is known that Round-Robin (which processes all alive jobs equally) is $2$-competitive and that this is the best competitive ratio one can hope for \cite{motwani1994nonclairvoyant}. 

We study this basic scheduling problem assuming certain predictions are available for use. Following the recent work by Lindermayr and Megow \cite{Megow22}, we will assume that we are given $k$ orderings/sequences as prediction, $\{\sigma_\ell\}_{\ell \in [k]}$. Each $\sigma_\ell$ is a permutation of $J := [n]$. Intuitively, it suggests an ordering in which jobs should be processed. This prediction is inspired by 
the aforementioned Shortest Job First (SJF) as an optimal schedule can be described as an ordering of jobs, specifically increasing order of job sizes. 

For each $\sigma_\ell$, its error is measured as 
$\eta(J,\sigma_\ell) := \cost(J, \sigma_\ell) - \opt(J)$,
where $\cost(J, \sigma_\ell)$ denotes the objective of the schedule where jobs are processed in the order of $\sigma_\ell$ and $\opt(J)$ denotes the optimal objective value. We may drop $J$ from notation when it is clear from the context.

As observed in \cite{Megow22},  the error  can be expressed as 
$ \eta(J,\sigma_\ell) = \sum_{i < j \in J} I^\ell_{i, j} \cdot |p_i - p_j|$,
where $I^\ell_{i, j}$ is an indicator variable for `inversion' that has value 1 if and only if $\sigma_\ell$ predicts the pairwise ordering of $i$ and $j$ incorrectly. That is, if $p_i < p_j$, then the optimal schedule would process $i$ before $j$; here $I^\ell_{i, j} =1$ iff $i \succ_{\sigma_\ell} j$.

As discussed in \cite{Megow22}, this error measure satisfies two desired properties, monotonicity and Lipschitzness, which were formalized in \cite{Im21spaa}. 

Our main result is the following. 
\begin{theorem}
    \label{thm:sched}
    Consider a constant $\eps  > 0$. Suppose that for any $S \subseteq J$ with $|S| = \Theta(\frac{1}{\eps^4}(\log \log n + \log k + \log (1 / \eps)))$, we have 
    $\opt(S) \leq c \eps \cdot \opt(J)$ for some small absolute constant $c$. Then, there exists a randomized algorithm that yields a schedule whose expected total completion time is at most 
    $(1+\eps) \opt + (1+\eps) \frac{1}{\eps^5} \eta(J, \sigma_\ell)$ for all $\ell \in [k]$.
\end{theorem}

As a corollary, by running our algorithm with $1 - \eps$  processing speed and simultaneously running Round-Robin with the remaining $\eps$ of the speed, the cost increases by a factor of at most $\frac{1}{1 - \eps}$ while the resulting hybrid algorithm is $2 / \eps$-competitive.\footnote{This hybrid algorithm is essentially the preferential time sharing \cite{Purohit,Im21spaa,Megow22}.
Formally, we run our algorithm ignoring RR's processing and also run RR ignoring our algorithm; this can be done by a simple simulation. 
Thus, we construct two schedules concurrently and each job completes at the time when it does in either schedule. This type of algorithms was first used in \cite{Purohit}.
}

\subsection{Algorithm}

To make our presentation more transparent we will first round job sizes. Formally, we choose $\rho$ uniformly at random from $[0, 1)$. Then, 
round up each job $j$'s size to the closest number of the form $(1+\eps)^{\rho + t}$ for some integer $t$. Then, we scale down all job sizes by $(1+\eps)^\rho$ factor.  We will present our algorithm and analysis assuming that every job has a size equal to a power of $(1+\eps)$. \submit{In the supplementary we}\full{Later we will} show how to remove this assumption without increasing our algorithm's objective by more than $1+\eps$ factor in expectation\submit{.}\full{ (Section~\ref{sec:sched-remove-assumption}).}

We first present the following algorithm that achieves Theorem~\ref{thm:sched} with $|S| = \Theta(\frac{1}{\eps^4} (\log n + \log k))$. The improved bound claimed in the theorem needs minor tweaks of the algorithm and analysis and they are deferred to the supplementary material. 

Our algorithm runs in rounds. Let $J_r$ be the jobs that complete in round $r \geq 1$. For any subset $S$ of rounds, $J_S := \cup_{r \in S} J_r$. For example, $J_{\leq r} := J_1 \cup \ldots \cup J_r$. 
Let $n_r := |J_{\geq r}| = n - |J_{< r}|$ denote the number of alive jobs at the beginning of round $r$. 

Fix  the beginning of round $r$.  The algorithm processes the job in the following way for this round. If $n_r \leq \frac{1}{\eps^4} (\log n + \log k)$, we run Round-Robin to complete all the remaining jobs, $J_{\geq r}$. This is the last round and it is denoted as round $L+1$. Otherwise, we do the following Steps 1-4: 

\paragraph{Step 1. Estimating $\eps$-percentile.} Roughly speaking, the goal is to estimate the $\eps$-percentile of job sizes among the remaining jobs. 
    For a job $j \in J_{\geq r}$, define its rank among $J_{\geq r}$ as the number of jobs no smaller than $j$ in $J_{\geq r}$ breaking ties in an arbitrary yet fixed way. Ideally, we would like to estimate the size of job of rank $\eps n_r$, but  do so only approximately.
        
    The algorithm will find $\tilde q_r$ that is the size of a job whose rank lies in $[ \eps(1  - \eps)n_r, \eps(1  + \eps)  n_r]$. 
    To handle the case that there are many jobs of the same size $ \tilde q_r$, we estimate $y_r$ the number of jobs no bigger than $\tilde q_r$; let $\tilde y_r$ denote our estimate of $y_r$. We will show how we can do these estimations without spending much time by    sampling some jobs and partially processing them in Round-Robin manner (the proof of the following lemma can be found in \full{Section~\ref{sec:sub1}.}\submit{the supplementary material.})

\begin{lemma}
    \label{lem:sched-step1}
     W.h.p. the algorithm can construct estimates $\tilde q_r$ and $\tilde y_r$ in time at most $O(\tilde q_r \frac{1}{\eps^2}\log n)$ such that there is a job of size $\tilde q_r$ whose rank lies in $[ \eps(1  - \eps)n_r, \eps(1  + \eps)  n_r]$ and $|\tilde y_r - y_r| \leq \eps^2 n_r$.
\end{lemma}

\paragraph{Step 2. Determining Good and Bad Sequences.} Let $\sigma^r_\ell$ denote $\sigma_\ell$ with all jobs completed in the previous rounds removed and with the relative ordering of the remaining jobs fixed. Let $\sigma^r_{\ell, \eps}$ denote the first $\tilde y_r$  jobs in the ordering. We say a job $j$ is big if $p_j > \tilde q_r$; middle if $p_j = \tilde q_r$; small otherwise. Using sampling and partial processing we will approximately distinguish good and bad sequences. Informally $\sigma^r_\ell$ is good if $\sigma^r_{\ell, \eps}$ has few big jobs and bad if it does many big jobs. The proof of the following lemma can be found in \full{Section~\ref{sec:sub2}.}\submit{the supplementary material.}

\begin{lemma}
    \label{lem:sched-step2}
        For all $\ell \in [k]$, we can label sequence $\sigma_\ell^r$ either good or bad in time at most $O(\tilde q_r \frac{1}{\eps^2}(\log n + \log k))$ that satisfies the following with high probability: If it is good, $\sigma^r_{\ell, \eps}$ has at most 
        $3\eps^2  n_r$ big jobs; otherwise $\sigma^r_{\ell, \eps}$ has at least $\eps^2  n_r$ big jobs.
\end{lemma}

\paragraph{Step 3. Job Processing.}  If all sequences are bad, then we process all jobs, each up to $\tilde q_r$ units in an arbitrary order. Otherwise, we process the first $\tilde y_r$ jobs in an arbitrary good sequence, in an arbitrary order,  each up to $\tilde q_r$ units.

\paragraph{Step 4. Updating Sequences.} The jobs completed in this round drop from the sequences but the remaining jobs' relative ordering remains fixed in each (sub-)sequence. For simplicity, we assume that partially processed jobs were never processed---this is without loss of generality as this assumption only increases our schedule's objective.

\full{
\subsubsection{Subprocedure: Step 1} \label{sec:sub1}

We detail the first step of the algorithm. Our goal is to prove Lemma~\ref{lem:sched-step1}.
To obtain the desired estimates, we take a sample $S_r$ of size $\frac{1}{\eps^2} \log n$ from $J_{\geq r}$ with replacement. The algorithm processes the sampled jobs using Round-Robin until we complete $\eps |S_r|$ jobs.\footnote{If a job is sampled more than once, we can pretend that multiples copies of the same job are distinct and simulate round robin \cite{Im21spaa}.} Note that there could be multiple jobs that complete at the same time.  This is particularly possible because we assumed that jobs have sizes equal to a power of $1+\eps$. In this case, we break ties in an arbitrary but fixed order. Let $\tilde q_r$ be the size of the job that completes $\eps |S_r|$th. 
If $m_r$ is the number of jobs in $S_r$ we completed, we estimate $\tilde y_r = \frac{m_r}{|S_r|}n_r$.

Let $B_1$ is the bad event that $\tilde q_r$ has a rank that  doesn't belong to $[ a_1 := \eps(1  - \eps)n_r, a_2 := \eps(1  + \eps)  n_r]$. Let $X_1$ be the number of jobs sampled that have rank at most $a_1$. Similarly let $X_2$ be the number of jobs sampled that have rank at most $a_2$. Note that $\neg B_1$ occurs if $X_1  \leq  \eps |S_r| \leq X_2$. Thus, we have $\Pr[B_1] \leq \Pr[ X_1 > \eps |S_r|] + \Pr [  X_2 < \eps |S_r|]$. Note that $X_1 = Binomial(a_1 / n_r, |S_r|)$. 

We use the following well-known Hoeffding's Inequality. 

\begin{theorem} \label{Hoeffding}
Let $Z_1 , \cdots , Z_T$ be independent random variables such that $Z_i \in [0, 1]$ for all $i \in [T]$. Let  $\bar{Z} = \frac{1}{T} (Z_1 + \cdots + Z_T)$.  We have $P(|\bar{Z} - E[\bar{Z}]| \geq \delta) \leq 2e^ {-2T\delta^2}$ where $\delta \geq 0$.
\end{theorem}

By Theorem~\ref{Hoeffding} we have, 
\begin{align*}
\Pr[ X_1 < \eps |S_r|] &= \Pr[ X_1 / |S_r| - a_1 / n_r <   \eps - a_1 / n_r = \eps^2] \leq 2 \exp(-2 |S_r| \eps^2) = 2 / n^2.
 \end{align*}
Similarly, we can show that 
\begin{align*}
\Pr[ X_2 > \eps |S_r|] &= 2 / n^2
\end{align*}
Thus, we conclude that $\Pr[B_1] \leq 4/ n^2$.

We consider the second bad event $|\tilde y_r - y_r| > \eps^2 n_r$, which we call $B_2$. Assume that $\tilde q_r$ is the size of a fixed job of rank in $[a_1, a_2]$. If $m'$ is the actual number of jobs of size $\tilde q_r$ in $J_{\geq r}$, $m_\ell = Binomial(m' / n_r, |S_r|)$. As before, using Theorem~\ref{Hoeffding} we can show that $\Pr[B_2] \leq 4 / n^2$. Thus, we can avoid both bad events simultaneously with probability $1 - 8 / n^2$.

Since we process each job in $S_r$ up to at most $\tilde q_r$ units, the total time we spend for estimation in Step 1 is at most $\tilde q_r \frac{1}{\eps^2}\log n$. This completes the proof.

\subsubsection{Subprocedure: Step 2} \label{sec:sub2}

We now elaborate on the second step. Our goal is to prove Lemma~\ref{lem:sched-step2}. To decide whether each sequence is good or not, as before we take a uniform sample from $S'_r$ of size $\frac{1}{\eps^2}(\log n + \log k)$ with replacement. By processing the jobs each up to $\tilde q_r$ units, we can decide the number of jobs in $S'_r$ of size no bigger than $\tilde q_r$. If the number is $c$, we estimate $\tilde y_r = \frac{c}{|S'_r|} n_r$. The proof follows the same lines as Lemma~\ref{lem:sched-step1} and is omitted. The only minor difference is that we need to test if each of the $k$ sequences is good or not, and to avoid the bad events for all $k$ sequences simultaneously, we ensure the probability that the bad events occur for a sequence is at most $O(\frac{1}{kn^2})$. This is why we use a bigger sample size than in Step 1. 
}

\subsection{Analysis of the Algorithm's Performance}
\submit{
We defer the analysis of the above algorithm (the proof of Theorem~\ref{thm:sched}) to the supplementary material, as it is quite technical and complex.  At a very high level, though, we use the fact that the error in each prediction can be decomposed into pair-wise inversions, and moreover we can partition the inversions into the rounds of the algorithm in which they appear.  Then we look at each round, and split into two cases.  First, if all sequences are bad then every prediction has large error, so we can simply use Round Robin (which is $2$-competitive against OPT) and the cost can be charged to the error of any prediction.  Second, if there is a good sequence, then in any good sequence the number of big jobs is small (so we do not spend much time processing them), and we therefore complete almost all of the non-big jobs. Here, we crucially use the fact that we can process the first $\eps$ fraction of jobs in a sequence in an arbitrary order remaining competitive against the sequence. Finally, we show that all of the additional assumptions and costs (e.g., rounding processing times and the cost due to sampling) only change our performance by a $1+\epsilon$ factor.  Getting all of these details right requires much care. 
}

\full{
Let $\sigma^*$ be an arbitrary sequence against which we want to be competitive. The analysis proceeds in rounds. 
Let $b_r$ be the start time of round $r$; by definition $b_1 = 0$. For the sake of analysis we decompose our algorithm's cost as follows: 
$$\sum_{r \in [L]} \left(\sum_{j \in J_r} (C_j - b_r) + T_r \cdot |J_{> r}|\right) + 2 \opt(J_{L+1}),$$
where $T_r$ is the total time spent in round $r$. To see this, observe that
$\sum_{j \in J_r} (C_j - b_r)$ is the total completion time of jobs that complete in round $r$, ignoring their waiting time before $b_r$. Each job $j \in J_r$'s waiting time in the previous rounds is exactly $\sum_{r' \in [L-1]} T_{r'}$. The total waiting time over all of the jobs during rounds where they are not completed is $\sum_{r' \in [L-1]} T_{r'} \cdot | J_{> r'}|$. The last term $2 \opt(J_{L+1})$ follows from the fact that we use Round-Robin to finish the last few jobs in the final round $L+1$.

To upper bound $A_r := \left(\sum_{j \in J_r} (C_j - b_r) + T_r \cdot |J_{> r}|\right)$ by $\cost(\sigma^*)$ we also decompose $\cost(\sigma^*)$ as
$$\sum_{r \in [L]} \left(\cost(\sigma^{*r+1}) - \cost(\sigma^{*r})\right) + \cost(\sigma^{*L+1})$$

Recall that we complete jobs $J_r$ in round $r$ and $\sigma$ evolves from $\sigma^r$ to $\sigma^{r+1}$ by dropping $J_r$ from the sequence $\sigma^r$. Thus, it decreases the cost of following $\sigma$. 

Since we use round-robin to finish the remaining jobs $J_{L+1}$ in the final round and round-robin is 2-competitive, our algorithm incurs cost at most $2 \opt(J_{L+1})$. If the assumption in Theorem~\ref{thm:sched} holds, this is at most $2 \eps \opt$. 

Let's take a close look at $\left(\cost(\sigma^{*r+1}) - \cost(\sigma^{*r})\right)$. This is equivalent to the following: 
For each pair $i \in J_r$ and $j \in J_{> r}$, $\sigma^*$ has an error $|p_i - p_j|$ if and only if it creates an inversion. This quantity can be charged generously. Let's call this aggregate error $\eta_r$.
But $\opt(J_r) + \sum_{i \in J_r, j \in J_{> r}} \min\{p_i, p_j\}$ should be charged sparingly. Let's call this part of optimum quantity $\opt_r$. Note that we are using the following: 

\begin{lemma}
    $\opt = \sum_{r \in [L]} \left(\opt(J_r) + \sum_{i \in J_r, j \in J_{> r}} \min \{p_i, p_j\}
    \right) + \opt(J^{L+1})$ 
\end{lemma}
\begin{proof}
We  know that the optimal schedule is Shortest-Job-First. Thus we have the following. Here we assume jobs are indexed in an arbitrary but fixed manner. 

\begin{eqnarray*}
\opt &=& \sum_{j \in J} p_j + \sum_{i \in J, p_i < p_j} p_i\\
&=&  \sum_{i \in J, j \in J, i\geq j} \min \{p_i,p_j\} \\
&=&  \sum_{r \in [L+1]}  \left (\sum_{i \in J_r, j \in J_r\, i\geq j} \min \{p_i,p_j\}      + \sum_{i \in J_r,j \in J_{> r}} \min \{p_i,p_j\} \right )\\
&=& \sum_{r \in [L+1]}  \left (\opt(J_r)      + \sum_{i \in J_r,j \in J_{> r}} \min \{p_i,p_j\} \right )\\
&=& \sum_{r \in [L]}  \left (\opt(J_r)      + \sum_{i \in J_r,j \in J_{> r}} \min \{p_i,p_j\} \right ) + \opt(J^{L+1}) \qedhere
\end{eqnarray*}
\end{proof}

Using the previous lemma, if we bound $A_r$ by $(1+O(\eps)) \opt_r + O(\frac{1}{\eps^5} )\eta_r$ we will have Theorem~\ref{thm:sched} by scaling $\eps$ appropriately. Proving this for all $r \in [L]$ is the remaining goal.

We consider two cases. 

\paragraph{All Sequences Are Bad.}
  In this case we complete all jobs of size at most $\tilde q_r$.
 Further $\sigma^*_{r, \eps}$ has at least $\eps^2 n_r$ big jobs. This implies that there are at least $\eps^2 n_r$  big jobs that do not appear in $\sigma^*_{r, \eps}$. So, for each pair of such a big job and a non-big job, $\sigma^*_r$ creates an inversion and has an error of at least $\eps \tilde q_r$, which contributes to $\eta_r$. Thus, $\eta_r \geq 
 \eps \tilde q_r \eps^4 n_r^2$.
 
 Now we want to upper bound $A_r$. Since the delay due to estimation is at most $\tilde q_r \frac{2}{\eps^2} (\log n + \log k)$ and all jobs are processed up to $\tilde q_r$ units, we have 
$A_r \leq (\tilde q_r \frac{2}{\eps^2} (\log n+ \log k)) * n_r + \tilde q_r (n_r)^2
\leq (2\eps +1) \tilde q_r (n_r)^2$. Thus, $A_r \leq O(1) \frac{1}{\eps^5} \eta_r$.

\paragraph{Some Sequences Are Good.} 
We will bound the expected cost of the algorithm for round $r$ when there is a good sequence.  The bad event $B_1$ occurs with very small  probability as shown in the analysis of Step 1. The contribution to the expected cost is negligible if the bad event occurs.  Due to this, we may assume that the event $\neg B_1$ occurs. 

 Say the algorithm processes the first $\tilde y_r$ jobs in a good sequence $\sigma_r$. By Lemmas~\ref{lem:sched-step1} and \ref{lem:sched-step2}, $J_r$ processes all small and middle jobs in $\sigma^r_\eps$. Additionally, the algorithm may  process up to $3 \eps^2 n_r$ big jobs without completing them. 

The total time it takes to process  the big jobs is $4 \eps^2 n_r \cdot \tilde q_r$ and up to $n_r$ jobs wait on them.  The contribution to the objective of all jobs waiting while these are processed is at most $4 \eps^2 n_r^2 \cdot \tilde q_r$.  We call this the wasteful delay due to processing big jobs. We  show that this  delay  is only $O(\eps)$ fraction of $A_r + A_{r+1}$.    

Consider $A_r + A_{r+1}$.  By definition of the algorithm, at least $\frac{1}{2} \eps n_r$ jobs of size at least $\tilde q_r$ are completed during rounds $r$ and $r+1$. 
 If less than $\frac{1}{2}\eps n_r$ middle or big jobs complete in round $r$, we can show that $n_{r+1} \geq (1 - 2\eps) n_r$. Then, we observe that $J_{\geq r+1}$ must have at most $4 \eps^2 n_r$ small jobs. 
This is because $\sigma^r_\eps$ includes at most $3\eps^2  n_r$ big jobs and it includes $\tilde y_r$ jobs. Since $y_r$ is the number of small and middle jobs and $|y_r - \tilde y_r| \leq \eps^2  n_r$, $\sigma^r_\eps$ must include all non-big jobs, except up to $4 \eps^2 n_r$. Thus, most jobs completing in round $r+1$ are middle or big and we can show that the number of such jobs completing in round $r+1$ is at least $\frac{1}{2} \eps n_r$. Therefore, at least $\frac{1}{2} n_r$ jobs will wait on the first $\frac{1}{2} \eps n_r$ jobs of size at least $\tilde q_r$ completed.  This implies, $A_r + A_{r+1} \geq \frac{1}{4}\tilde q_r  \eps n_r^2$.  This is at least a $\Theta(\frac{1}{\eps})$ factor larger than the wasteful delay.

Note that the delay due to processing big jobs and as well as the time spent computing the estimates (used in Lemma~\ref{lem:sched-step1} and \ref{lem:sched-step2}) is at most $n_r \cdot 6\eps^2 n_r \tilde q_r \leq O(\eps) \cdot (A_r + A_{r+1})$. In the following, we will bound the cost without  incorporating these costs.   Factoring in these two costs will increase the bound by at most $1 / (1 - O(\eps))$ factor.

Thus, we only need to consider small and middle jobs that complete.   For the sake of analysis, we can consider the worst case scenario where we first process middle jobs and then small jobs. We will bound $A_r$ by $\opt(J_r)$ (i.e. without charging to $\eta_r$). In this case,  $A_r / \opt_r$ is maximized when when all small jobs have size 0. We will assume this in the following.
Let $m$ be the number of mid sized jobs we complete. For brevity, assume that the number of small jobs is at most $\eps n_r$ although the actual bound is $\eps ( 1+ \eps) n_r$. Further, we assume that $m(m+1) \approx m^2$ as we are willing to lose $(1+\eps)$ factor in our bound. Let $n = n_r$ for notational convenience. Then, we have $A_r = m^2 \tilde q_r / 2+ m  \tilde q_r \eps n + (n(1 - \eps) - m) m  \tilde q_r = m^2  \tilde q_r / 2 + (n - m)m  \tilde q_r$. In contrast, $\opt_r = m^2  \tilde q_r /2 + ((1 - \eps) n - m) m  \tilde q_r$. The ratio of the two quantities is $\frac{n - m / 2}{(1 - \eps)n - m / 2}$ where $m \leq n$. Therefore, in the worst case $\frac{1/2}{1/2 - \eps} \leq ( 1+ 3\eps)$. Thus,  $A_r$ can be bounded by  $(1+3\eps) \opt_r$.

\subsection{Improved Guarantees}
    \label{sec:sched-improved}

One can show that at least $(\eps - 4\eps^2)$ fraction of jobs complete in every round assuming no bad events occur: if all sequences are bad then all non-big jobs complete, which means at least $(\eps - \eps^2) n_r$ jobs complete due to Lemma~\ref{lem:sched-step1}. Otherwise, any good sequence $\sigma^\ell_{r}$ has at least $(\eps - \eps^2 - 3 \eps^2)n_r$ non-big jobs in $\sigma^\ell_{r, \eps}$ due to Lemmas~\ref{lem:sched-step1} and \ref{lem:sched-step2} and they all complete if $\sigma^\ell_{r}$ is chosen. 
Thus there are at most $O(\frac{1}{\eps} \log n)$ rounds and there are at most $O(\frac{k}{\eps} \log n)$ bad events, as described in Lemmas~\ref{lem:sched-step1} and \ref{lem:sched-step2}, to be avoided. 

Suppose we run round robin all the time using $\eps$ of the speed, so even in the worst case we have a schedule that is  $ 2/ \eps$-competitive against the optimum and therefore against the best prediction as well. This will only increase our upper bound by $1 / (1-\eps)$ factor. Then, we can afford to have bad events with higher probabilities. 

Specifically, we can reduce the sample size in Steps 1 and 2 to $s := \Theta(\frac{1}{\eps^2} \log (k (1 / \eps^3) \log n))$ to avoid all bad events with probability at least $ 1- \eps^2$; so if any bad events occur, at most an extra $(2/ \eps) \cdot \eps^2 \opt$ cost occurs in expectation. Then, we can show that we can do steps 1-4 as long as $n_r = \Omega(\frac{1}{\eps^4}(\log k + \log \log n + \log \frac{1}{\eps}))$.

 Then the delay due to estimation is at most $2s \tilde q_r n_r$ in Steps 1 and 2. We want to ensure that this is at most $O(\eps)$ fraction of $A_r + A_{r+1}$. Recall that we showed $A_r + A_{r+1} \geq \frac{1}{4}\tilde q_r  \eps n_r^2$. Thus, if we have $\frac{1}{4} \eps^2 \tilde q_r n_r^2 \geq 2 s \tilde q_r n_r$, we will have the desired goal. This implies that all the delay due to estimation can be charged to $O(\eps)$ of the algorithm's objective as long as $n_r \geq 8 \frac{1}{\eps^2} s$. This is why we switch to round-robin if $n_r \leq 8 \frac{1}{\eps^2} s$.

\subsection{Removing the Simplifying Assumption on Job Sizes}
\label{sec:sched-remove-assumption}

Recall that we chose $\rho$ uniformly at random from $[0, 1)$ and  
rounded up each $j$'s size to the closest number of the form $(1+\eps)^{\rho + t}$ for some integer $t$. Although we then scaled down all job sizes by $(1+\eps)^\rho$, we assume that we didn't do it. This assumption is wlog as all the bounds remain the same regardless of uniform scaling of job sizes.

Let $\eta^\ell$ be the prediction error of $\sigma^\ell$. Let $\bar \eta^\ell$ be the error after rounding up job sizes. Let $\bar p_j$ be $j$'s size after rounding. Note that i) $p_j \leq \bar p_j \leq (1+\eps) p_j$; ii) if $p_i \leq p_j$, then $\bar p_i \leq \bar p_j$. The second property implies jobs relative ordering is preserved. 

In the following we drop $k$ for notational convenience. We have 
\begin{align*}
    \eta := \sum_{i \neq j \in J} \bone(p_i < p_j) \cdot \bone(i \succ_{\sigma^*} j) \cdot |p_i - p_j| \\
    \bar \eta := \sum_{i \neq j \in J} \bone(p_i < p_j) \cdot \bone(i \succ_{\sigma^*} j) \cdot |\bar p_i - \bar p_j| 
\end{align*}

\begin{lemma}
    $\E \bar \eta = \Theta(1) \cdot \eta$.
\end{lemma}
\begin{proof}
    Thanks to linearity of expectation it suffices to show that $\E |\bar p_i - \bar p_j| = \Theta(1) \cdot |p_i - p_j|$ for every pair of jobs $i$ and $j$. 

Assume $p_i > p_j$ wlog. 
Scale both job sizes for convenience, so $p_j = 1$.  Let $p_i / p_j = (1+\eps)^\delta$. 
    
Case 1. $\delta \geq 1$. In this case we have $\bar p_i > \bar p_j$ almost surely. Using the fact that rounding up increases job sizes by at most $1+\eps$ factor, we can show that $(1/ 3) |p_i - p_j| \leq |\bar p_i - \bar p_j| \leq (1+\eps)^2|p_i - p_j|$. 

Case 2. $\delta \in (1/2, 1]$. In this case we can show that $|p_i - p_j|  = \Theta(\eps)$ and $\E |\bar p_i - \bar p_j| = \Theta(\eps)$.

Case 3. $\delta \in (0, 1/2)$. We have $|p_i - p_j| = (1+\eps)^\delta - 1$ and $\E |\bar p_i - \bar p_j| = \int^\delta_{\rho = 0}\eps (1+\eps)^\rho d\rho  = \frac{\eps}{\ln(1+\eps)} ((1+\eps)^\delta - 1) $. Thus, the ratio of the two is $\frac{\eps}{\ln(1+\eps)} = \Theta(1)$ for small $\eps > 0$.
\end{proof}

What we showed was the following. For any $\bar \eta^\ell$,
$\E A \leq (1+\eps) (\opt + O(1) \frac{1}{\eps^5} \bar \eta^\ell) + 2 (1+\eps)\opt(J_{K+1})$. 
By taking expectation over randomized rounding, we have
$\E A \leq (1+\eps)^2 (\opt + O(1) \frac{1}{\eps^5}  \eta^\ell) + 2 (1+\eps)^2\opt(J_{K+1})$. By scaling $\eps$ appropriately, we obtain the same bound claimed in Theorem~\ref{thm:sched}.

}

\subsection{Learning \texorpdfstring{$k$}{k} Predicted Permutations}

Now we show that learning the best $k$ permutations has polynomial sample complexity.  

\begin{theorem}
Let $\cD$ be an unknown distribution of instances on $n$ jobs.  Given $S$ independent samples from $\cD$, there is an algorithm that outputs $k$ permutations $\spred_1,\spred_2,\ldots,\spred_k$ such that
$\E_{J \sim \cD} \left[ \min_{\ell \in [k]}\eta(J,\spred_\ell) \right] \leq \min_{\sigma_1,\sigma_2,\ldots,\sigma_k} \E_{J \sim \cD} \left[ \min_{\ell \in [k]}\eta(J,\sigma_\ell) \right] + \eps$ with probability $1-\delta$, where $S = \poly(n,k,\frac{1}{\epsilon},\frac{1}{\delta})$.
\end{theorem}
\begin{proof}
The algorithm is basic ERM, and the polynomial sample complexity follows from 
\submit{Theorem~\ref{thm:sample_complexity_multiple_predictions}}
\full{Theorem~\ref{thm:sample_complexity_multiple_predictions_app}}
and Theorem 20 in \citet{Megow22}.  
\end{proof}

\section{Conclusion}
Despite the explosive recent work in algorithms with  predictions, almost all of this work has assumed only a single prediction.  In this paper we study algorithms with \emph{multiple} machine-learned predictions, rather than just one.  We study three different problems that have been well-studied in the single prediction setting but not with multiple predictions: faster algorithms for min-cost bipartite matching using learned duals, online load balancing with learned machine weights, and non-clairvoyant scheduling with order predictions.  For all of the problems we design algorithms that can utilize multiple predictions, and show sample complexity bounds for learning the best set of $k$ predictions.  Demonstrating the effectiveness of our algorithms (and the broader use of multiple predictions) empirically is an interesting direction for further work.

Surprisingly, we have shown that in some cases, using multiple predictions is essentially ``free.'' For instance, in the case of min-cost perfect matching examining $k = O(\sqrt{n})$ predictions takes the same amount of time as one round of the Hungarian algorithm, but the number of rounds is determined by the quality of the {\em best} prediction. In contrast, for load balancing, using $k$ predictions always incurs an $O(\log k)$ cost, so using a constant number of predictions may be best. More generally, studying this trade-off between the cost and the benefit of multiple predictions for other problems remains an interesting and challenging open problem.

\bibliographystyle{plainnat}
\bibliography{references}

\newpage 
\appendix

\section{Experiments} \label{sec:exps}

We now consider a preliminary empirical investigation of the algorithm proposed in Section~\ref{sec:matching} for min-cost perfect matching with a portfolio of predicted dual solutions.  For this, we modify the experimental setup utilized by Dinitz et al.~\cite{DinitzILMV21} to evaluate the case of single prediction for this problem.  We use a training set which is a mixture between three distinct types of instances, and show that by using more than one prediction we see an overall improvement.

\begin{table}[h]
    \centering
    \begin{tabular}{c|ccccc}
        Dataset & Shuttle & KDD & Skin~\cite{uci_skin}  \\ 
        \hline
        \# of Points ($n$) & 43500 & 98,942 & 100,000  \\
         \# of Features ($d$) &  10 & 38 & 4  \\
    \end{tabular}
    \caption{Datasets used in preliminary experiments.}
    \label{tab:datasets}
\end{table}

\noindent \textbf{Experiment Setup and Datasets:} These experiments were performed on a machine with a 6 core, 3.7 GHz AMD Ryzen 5 5600x CPU and 16 GB of RAM.  The algorithms were implemented in Python and the code is available at \url{https://github.com/tlavastida/PredictionPortfolios}.

To construct bipartite matching instances we adopt the same technique as Dinitz et al.~\cite{DinitzILMV21} for Euclidean data sets.  At a high level, their technique takes in a dataset $X$ of points in $\R^d$ and a parameter $n \in \N$, and outputs a distribution $\cD(X,n)$ over dense instances of min-cost perfect matching with $n$ nodes on each side.  Sampling from the distribution $\cD(X,n)$ can be done efficiently.  We consider three datasets (Shuttle, KDD, and Skin) from the UCI Machine Learning Repository~\cite{ucimlr} that were also considered in \cite{DinitzILMV21}, see Table~\ref{tab:datasets} for details.  

Instead of considering each dataset (and its corresponding distribution) separately, we consider them together as a mixture model.  For each dataset, we consider a sub-sample $X$ of 20,000 points and we set $n=150$ (so $2n = 300$ nodes per instance) in order to construct each distribution $\cD(X,n)$.  We then sample 20 instances from each distribution and consider these 60 instances together as our training set (i.e. our learning algorithm doesn't know which dataset each instance was derived from).  For testing, we sample an additional 10 instances from each dataset.

\noindent \textbf{Results:} To evaluate our approach, we vary the number of predicted dual solutions ($k$) learned from one (baseline) to five and also compare to the standard Hungarian method (referred to as ``No Learning''). Following \cite{DinitzILMV21} our evaluation metrics are running time and the number of iterations of the Hungarian algorithm.

Given that the dataset has three distinct clusters by construction, we expect the average number of iterations of the Hungarian algorithm to decrease as $k$ grows from 1 to 3 and then stay stable.  We also expect the running time to decrease as $k$ grows from 1 to 3, and then increase as $k$ grows from 3 to 5, as the cost of the projection step grows linearly with $k$.

In Table\ref{tab:results} we have divided our results on the test instances by dataset, so that different scales don’t obscure the results. 

\begin{table}[h]
    \centering
    \begin{tabular}{c|cc|cc|ccc}
        Dataset  & \multicolumn{2}{c}{Shuttle} & \multicolumn{2}{c}{KDD} & \multicolumn{2}{c}{Skin~\cite{uci_skin}}   \\ \hline
        $k$ & \# Iterations & Time (s) & \# Iterations & Time (s)  & \# Iterations & Time (s)   \\
        \hline
        No Learning & 149.0 & 0.929 & 304.4 & 1.97 & 63.1 & 0.396 & \\
         1 &  77.9 & 0.601 & 149.4 & 1.09 & 68.7 & 0.472 \\
         2 &  59.0 & 0.433 & 144.0 & 1.08 & 306.7 & 2.690 \\
         3 &  42.4 & 0.350 & 144.0 & 1.09 & 38.9 & 0.318 \\
         4 &  42.4 & 0.373 & 130.3 & 1.05 & 38.9 & 0.342 \\
         5 &  42.4 & 0.394 & 136.1 & 1.15 & 38.9 & 0.357 \\
    \end{tabular}
    \caption{Average iteration count and average running time across each data set and value of $k$.}
    \label{tab:results}
\end{table}

Observe that for the Shuffle dataset the process proceeds exactly as we had predicted – the number of iterations drops as $k$ grows from 0 to 3 and then stays constant, the wall clock time drops as well and then starts increasing.

For KDD and Shuffle the situation is more interesting. For KDD we don’t see the nice inflection point at $k=3$. We conjecture that the KDD dataset itself is diverse and induces many subclusters for its family of matching instances, thus increasing the number of clusters keeps on improving the performance.

For the Skin dataset, we observe an anomaly at $k=2$. We conjecture that this is due to the clustering in the learning stage allocating both predictions to the KDD and Shuffle datasets, essentially ignoring the Skin dataset and thus giving extremely poor performance. In other words, the gain from allocating the “extra” prediction to the other datasets was enough to outweigh the cost to the Skin data. This aligns well with our conjecture that the KDD data itself has many subclusters.

Overall, however, we see that the empirical evaluation supports our conclusion that there are performance gains to be had when judiciously using multiple advice models.

\end{document}